\documentclass[runningheads]{llncs}

\usepackage{amsthm}
\usepackage[T1]{fontenc}
\usepackage{graphicx}
\usepackage{microtype}

\usepackage{caption}
\usepackage{subcaption}

\usepackage{booktabs} 
\usepackage{multirow}
\usepackage{makecell}
\usepackage{placeins}

\usepackage{hyperref}

\usepackage{amsmath}
\usepackage{amssymb}
\usepackage{mathtools}
\usepackage{bm}
\usepackage{paralist}
\usepackage[capitalize,noabbrev]{cleveref}

\theoremstyle{plain}

\def\R{\mathbb{R}}

\DeclareMathOperator*{\argmax}{arg\,max}

\def\softmax{\mathrm{softmax}}
\def\our{\operatorname{t-softmax}}

\def\ourr{\operatorname{r-softmax}}

\author{Klaudia Bałazy \and
    Łukasz Struski \and
    Marek Śmieja \and
    Jacek Tabor}

\institute{Jagiellonian University \\
\texttt{Corresponding author: klaudia.balazy@doctoral.uj.edu.pl}}

\begin{document}
%
\title{\boldmath$\operatorname{r-softmax}$: Generalized Softmax with Controllable Sparsity Rate}
%
%


%
%
%
\maketitle              

\begin{abstract}
Nowadays artificial neural network models achieve remarkable results in many disciplines. Functions mapping the representation provided by the model to the probability distribution are the inseparable aspect of deep learning solutions. Although softmax is a commonly accepted probability mapping function in the machine learning community, it cannot return sparse outputs and always spreads the positive probability to all positions. In this paper, we propose r-softmax, a modification of the softmax, outputting sparse probability distribution with controllable sparsity rate. In contrast to the existing sparse probability mapping functions, we provide an intuitive mechanism for controlling the output sparsity level. We show on several multi-label datasets that r-softmax outperforms other sparse alternatives to softmax and is highly competitive with the original softmax. We also apply r-softmax to the self-attention module of a pre-trained transformer language model and demonstrate that it leads to improved performance when fine-tuning the model on different natural language processing tasks.
\end{abstract}



\section{Introduction}
\label{intro}

Deep learning models achieve state-of-the-art results in various domains such as computer vision, natural language processing (NLP), chemical sciences, and many others. Transforming the numerical output, returned by a neural network into a probability distribution on a discrete set is an integral aspect of many machine learning models. In classification, it describes the probability over classes; in the attention mechanism for NLP, it indicates which words in a text are contextually relevant to other words. The generally accepted standard for probability mapping function is a softmax function~\cite{bridle1990probabilistic,luce2012individual}. Softmax is easy to evaluate and differentiate as well as it can be transformed into convex a loss function, which is especially appealing in classification problems. 

Although softmax is the most widely applied probability mapping function in machine learning, it cannot return sparse outputs. In other words, softmax assigns a non-zero probability to every component. The representation that allows for zero probabilities would be more natural and more interpretable as certain elements could be clearly marked as insignificant. Since softmax always spreads the positive probability to all positions, it does not return the number of relevant labels, i.e. those with non-zero probabilities. In consequence, applying softmax function in multi-label classification involves defining a threshold below which the label is considered negative, which requires the hyperparameter selection process that generates additional computational overhead. 

\begin{figure}[tb!]
\vskip 0.2in
\begin{center}    \includegraphics[scale=0.34]{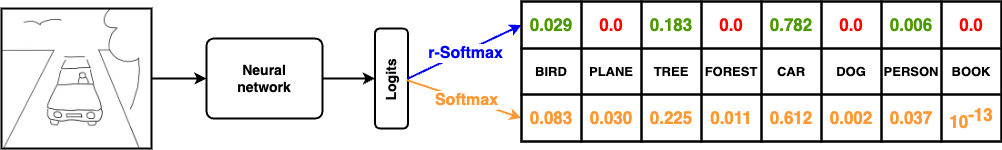}
\end{center}
\caption{The difference between using softmax and $\ourr$ for multi-label classification. Both functions return the probability distribution over the specified classes based on the output provided by the neural network model. Since $\ourr$ is able to produce zero probabilities, we can consider them as an indication of a negative class. For softmax, we need to select an appropriate threshold below which a class will be classified as negative. Thus, the representation provided by $\ourr$ is more intuitive and more interpretable. \label{fig:intro}}
\vskip -0.2in
\end{figure}

In this paper, we introduce $\ourr$, a sparse alternative to softmax function, that eliminates the problem of non-zero probabilities and allows for the intuitive control of the sparsity rate. The sparsity rate~$r$, representing the fraction of desired zero values, can be specified by the user, as well as the model can be trained to select its appropriate value using a typical gradient descent procedure. In consequence, applying $\ourr$ in multi-label classification and training a model to predict appropriate~$r$, eliminates the need for defining an additional mechanism, e.g. a threshold, for deducing the number of positive labels, see \Cref{fig:intro}. 

We evaluate $\ourr$ as a function determining probabilities of classes in a multi-label classification problem and as a function determining the significance probability of elements in the attention mechanism. In the multi-label classification scenario, $\ourr$ is benchmarked on various synthetic and real datasets. Our experiments demonstrate that the performance of $\ourr$ is significantly better than other sparse alternatives to softmax, like sparsemax~\cite{sparsemax} and sparsehourglass~\cite{laha2018controllable}, and is competitive with the original softmax with a selected optimal threshold determining if the label is positive. In the case of the attention mechanism, we replace softmax mapping with $\ourr$ in the pre-trained transformer language model. We show that our modification can improve the performance of the fine-tuned model on various NLP tasks. 

Our contribution can be summarized as follows:
\begin{itemize}
    \item We introduce $\ourr$, a sparse probability mapping function that is a generalization of the original softmax. The desired sparsity rate $r$ can be defined by the user or learned by the model itself.
    \item We provide an extensive evaluation of $\ourr$ on the multi-label classification problem that demonstrates the benefits of using our method.
    \item We show that replacing softmax with $\ourr$ in the pretrained transformer language model improves the performance of a fine-tuned model on most of the considered NLP tasks.
\end{itemize}

\section{Related Work}
\label{related}

Functions mapping the output of an artificial neural network into the probability distribution are indispensable components in machine learning. They are useful, for example, in determining class membership in a classification problem or in assessing the significance of the elements under consideration. 

\subsubsection{Softmax} Softmax is a commonly used function in machine learning, which parametrizes a probability distribution over a discrete set of outputs~\cite{bridle1990probabilistic,luce2012individual}. Its application ranges from classification through attention mechanism~\cite{vaswani2017attention} to reinforcement learning~\cite{sutton2018reinforcement}. However, softmax cannot return sparse outputs with zero values at certain positions. In consequence, in multi-label classification, we need to find a threshold under which the label is considered negative.

However, classification models with softmax frequently return overconfident predictions, which exceed model accuracy resulting in uncalibrated models~\cite{guo2017calibration}. Moreover, softmax rarely spreads similar probability to a few positions, which is in particular inconvenient in multi-label classification, where more than one label per example may be correct. Another disadvantage is caused by the fact that softmax cannot return sparse outputs with zero values at certain positions. In consequence, in multi-label classification, we need to find a threshold under which the label is considered negative. Moreover, non-sparse outputs generate computational overhead in the case of high-dimensional outputs.

\subsubsection{Alternatives to softmax} Given the broad range of applications for probability mapping functions in machine learning, various alternatives to softmax have been developed, each with its own set of benefits and drawbacks depending on the particular use case. Noteworthy alternatives to softmax include the spherical softmax~\cite{de2015exploration}, multinomial probit~\cite{albert1993bayesian}, softmax approximations~\cite{bouchard2008efficient} or Gumbel-Softmax~\cite{jang2016categorical}, which provides a continuous probability distribution that serves as an approximation of the discrete distribution produced by softmax. As our paper introduces a novel sparse alternative to softmax, below we focus on existing sparse probability mapping functions. 

Sparsemax~\cite{sparsemax} is defined as a projection of the input vector onto the probability simplex. Since the projection is very likely to hit the boundary of the simplex, sparsemax returns sparse outputs. The authors also constructed a natural convex loss for sparsemax, making an analogy with a derivative of the cross-entropy loss applied to softmax. Although the derivation of the model is theoretically justified, its performance is usually inferior to softmax models.

In \cite{laha2018controllable}, the authors defined a general family of probability mapping functions, which includes many popular functions, such as softmax or sparsemax, as special cases. By adding a regularization term and component-wise transformation function to the sparsemax, they constructed a general formulation of probability mapping functions. They also proposed a general strategy of designing convex loss functions for their models, including an alternative loss for sparsemax, which increased its experimental performance. A theoretical contribution of the paper is further enriched by the formulating desirable properties for probability mapping functions.


\section{Sparse version of softmax}
\label{sec.theorem}

In this section, we introduce $\ourr$, a sparse probability mapping function with a controllable sparsity rate. First, we describe the motivation behind the use of the sparse mapping function. Next, we define the weighted softmax -- a generalization of the classical softmax~\cite{bridle1990probabilistic}. Finally, we introduce $\ourr$, where the sparsity rate can be easily defined by the user.

\subsubsection{Problem motivation} Probability mapping function is a key component in typical deep learning applications. It allows for transforming a real-valued response ${x = (x_1,\ldots,x_n) \in \R^n}$ of the neural network to the probability vector $p = (p_1,\ldots,p_n)$, where $p_i \geq 0$ and $\sum_{i=1}^n p_i = 1$. To parameterize this probability, we usually use the softmax function:
\[
\softmax{(x)} = \Big(\tfrac{\exp(x_1)}{\sum \limits^n_{i=1}\exp(x_i)}, \dots, \tfrac{\exp(x_n)}{\sum \limits^n_{i=1} \exp(x_i)}\Big).
\]
Since softmax is in fact the normalized exponential function, it can be evaluated and differentiated efficiently, which makes it very appealing in training deep learning models. To discuss a specific softmax application, let us consider a classification problem. In this case, the component $p_i$ describes the probability that the input example comes from the $i$-th class. If we know that every example has a single class label, then we return a class with maximal probability:
$$
\mathrm{class}(x) = \arg \max_i p_i.
$$
If more than one class can be correct for a given example (multi-label classification), we return $k$ classes with the highest probabilities. There appears a natural question of {\em how to select the number of classes $k$ for a given input?} Since the softmax function does not return zero probabilities, we cannot easily say what probability should be converted to a positive label and which should not. In consequence, we arrive at a problem of manually introducing a threshold below which the class label will be considered negative.

The above example illustrates the basic problem with softmax that it cannot return sparse outputs. If the probability mapping function would be able to zero out probabilities, then we could interpret zero probabilities as negative labels and the remaining ones as positive labels. This requirement is also important for other machine learning problems. The main building block of recent transformer architecture~\cite{vaswani2017attention} is a self-attention layer, which is responsible for selecting key information from a given representation. By applying softmax, we force the model to consider all components as relevant, which usually is not the case. The attention module should be able to ignore unnecessary information by assigning zero probability to selected components.
 
\subsubsection{The weighted softmax} Keeping the above motivation in mind, we focus on constructing an alternative to softmax mapping, which is capable of returning sparse output vectors. We first define the weighted softmax~--~a general form of the probability mapping function. By a proper parameterization of its weights, the weighted softmax can reduce to a typical softmax, or binary one-hot vector, in which the coordinate containing maximal probability is rounded to 1 and the remaining coordinates are clipped to 0. It can also parametrize sparse probability mapping functions, which lay between softmax and one-hot vectors.

Let $x=(x_1,\ldots,x_n) \in \R^n$ be a point, associated with vector of weights $w=(w_1,\ldots,w_n) \in \R_+^n$, where $\sum^n_{i=1}w_i > 0$. We define a weighted softmax by the following formula:
\[
\softmax{(x, w)} = \Big(\tfrac{w_1\exp(x_1)}{\sum \limits^n_{i=1}w_i\exp(x_i)}, \dots, \tfrac{w_n\exp(x_n)}{\sum \limits^n_{i=1}w_i\exp(x_i)}\Big).
\]
All components of the weighted softmax are non-negative and sum to~$1$, which means that it is a proper parametrization of a discrete probability distribution. For a constant weight vector $w$, the weighted softmax reduces to classical softmax. A crucial difference between softmax and weighted softmax is that the weighted softmax is able to return zeros at some coordinates. To zero out the $i$-th coordinate it is enough to set $w_i=0$. In the extreme case, the weighted softmax can produce one-hot vectors by setting exactly one non-zero weight.

We are interested in such a parametrization of weights in the weighted softmax, which allows for a smooth transition between softmax and binary one-hot vectors. For this purpose, we construct $\our$, in which all weights depends on a single parameter~$t>0$:
\begin{equation}\label{eq.t-softmax}
\our(x, t) = \softmax(x, w_t),
\end{equation}
where $w_t = (w^1_t, \ldots, w^n_t)$ and $w^i_t = \mathrm{ReLU}(x_i + t - \max(x))$.
Clearly, all weights~$w_i$ are nonnegative and there is at least one positive weight, which is consistent with the definition of weighted softmax. We can observe that the $i$-th weight is zero if the absolute difference between $x_i$ and the maximum value $\max(x)$ is greater than or equal to $t$.



The following examines how $\our$ changes with varying values of $t$:
\begin{theorem}
Let $x\in\R^n$ be a data point and let $t\in(0, \infty)$. Then
\begin{itemize}
    \item the limit of $\our(x, t)$ is $\softmax(x)$ as $t$ approaches infinity,
    \item if $x$ reaches unique max at index $k$, then
    \begin{equation} \label{eq:1}
        \our(x, t)=\mathrm{onehot}(\argmax_i(x)),
    \end{equation}
    for $t\in(0, x_k -\max_{i\neq k}(x)]$, where $\mathrm{onehot}(i)\in\R^n$ is a vector consisting of zeros everywhere except $k$-th position where $1$ is located.
\end{itemize}
\end{theorem}

\begin{proof}
The first property is a consequence of $\our(x, t) = \softmax(x, \tfrac{w_t}{t})$, and if $t$ approaches infinity then $\tfrac{w_t}{t}$ goes to $1$, leading to $\softmax(x, 1) = \softmax(x)$. The last property follows directly from the definition of $\our$.
\end{proof}

In practice, we can treat $t$ as a model parameter, which will be tuned together with the remaining parameters in a training phase. This strategy is especially useful in a multi-label classification because we cannot decide a priori what is the correct number of positive labels for a given example. In this case, the model predicts both the number of positive labels as well as the distribution over classes. Experimental results show that this strategy gives promising results. 

\subsubsection{Controlling the number of non-zero values using r-softmax} 

Instead of learning the optimal value of $t$ as discussed above, there are situations in which we would like to have the ability to explicitly decide how many components returned by $\our$ should be zero. For this purpose, we introduce a parameter~$r\in[0, 1]$ that we call a \textit{sparsity rate}. Sparsity rate $r$ is an intuitive parameter that will represent the fraction of zero components we would like to obtain in the output probability distribution.

Recall that $w_i^t=0$ for $i=1,\ldots,n$ if $|x_i-\max(x)|\geq t$, as defined in~\Cref{eq.t-softmax}. To control the number of non-zero weights, we can inspect the range $[\min(x),\max(x)]$ and select $t$ such that $x_i < t < x_j$, where $x_i$ and $x_j$ are two distinct elements in ${x_1,\ldots,x_n}$, in increasing order. This will zero out the $i$-th component while keeping the $j$-th component non-zero. We can use the quantile of the set of $x$'s coordinates ${x_1,\ldots,x_n}$ to implement this rule. The $q$-quantile $\mathrm{quantile}(x, q)$ outputs the value $v$ in $[\min(x), \max(x)]$ such that the probability of ${x_i: x_i \leq v}$ equals $q$. If the quantile lies between $x_i$ and $x_j$ with indices $i$ and $j$ in the sorted order, we use linear interpolation to compute the result as $x_i + \alpha\cdot(x_j - x_i)$, where $\alpha$ is the fractional part of the computed quantile index. Setting $q=0$ or $q=1$ in $\mathrm{quantile}(x, q)$ will return the lowest or highest value of $x$, respectively.


Following the above motivation, we fix the sparsity rate $r \in [0,1]$ to quantify the requested fraction of zeros in a probability mapping function. The $r$-softmax is defined by:
\begin{equation}\label{eq.r_softmax}
    \ourr{(x, r)} = \our(x, t_r).
\end{equation}
where
$$
t_r = -\mathrm{quantile}(x, r) + \max(x).
$$
The above parameterization of $t_r$ determines that the fraction of $r$ components will be zero. In particular, applying $\ourr(x,r)$ on $x=(x_1,\ldots,x_n) \in \R^n$ and $r=\frac{k}{n}$, for $k \leq n$, will output a probability distribution with $k$ zero coordinates. 

Using $\ourr$ function allows to reduce the model complexity and eliminate less probable components. Experiments demonstrate that this mechanism is beneficial for example in the self-attention mechanism applied in NLP tasks.

\subsubsection{Summary} In summary, we propose a new function that maps an input to a sparse probability distribution. Our function has two versions (1) the $\our$ version (see~\Cref{eq.t-softmax}), which produces an output with a sparsity level guided by the parameter $t$ that can be learned automatically during model training through backpropagation (no need to select it manually), and (2) the $\ourr$ version (see~\Cref{eq.r_softmax}), which introduces an intuitive parameter $r$ that allows the user to specify the desired fraction of zero elements in the output. The parameter $r$ may be learned through backpropagation~(as we demonstrate in the multi-label classification experiments in~\Cref{sec:expmulti}) as well as it can be manually chosen by the user~(as we show in the self-attention experiments in~\Cref{sec:expattn}). It is worth to note that while the use of the $r$ parameter in $\ourr$ offers interpretability and control over the model's behavior, it comes with an increased computational cost due to the need to calculate the $t$ parameter using the $quantile$ function, which requires sorting the input vector. Therefore, when computational complexity is a concern, the $\our$ version may be a more suitable option than the $\ourr$ version.

\section{Experiments}
\label{exps}

In this section, we benchmark $\ourr$ function against the basic softmax and other sparse probability mapping functions such as sparsemax and sparsehouglass.

First, we consider the multi-label classification problem and show that $\ourr$ is in most cases the best probability mapping function. Next, we fine-tune a pre-trained language model with different functions applied in self-attention blocks and show that $\ourr$ is the most beneficial choice\footnote{Code with $\ourr$ is available at \url{https://github.com/gmum/rsoftmax}}.


\subsection{Alternative to softmax in multi-label classification}
\label{sec:expmulti}

The multi-label classification problem is an important problem that arises in many domains. For example, the image classification problem, where describing an image by a single class is often not sufficient as it usually consists of objects belonging to different classes~\cite{pascalvoc2007,kumar2021multilabel,cocodataset}. The last element of the architecture, in multi-label classification models, is typically a function that maps the output of the network to a vector representing the probability of belonging to different classes~\cite{Wang_2016_CVPR}. In many cases, this function is softmax~\cite{Wang_2016_CVPR}, but many other functions are also investigated, such as those that introduce sparse probability distributions~\cite{laha2018controllable,sparsemax}.

\subsubsection{R-softmax for multi-label classification}
To use $\ourr$ in multi-label classification, we need to select a proper loss function. Unfortunately, we cannot directly apply cross-entropy loss as $\ourr$ can return zeros for certain positions, which makes the log function undefined. To resolve this issue, we follow the reasoning used in~\cite{laha2018controllable}. For this purpose, let $z$ denote the logits returned by a neural network for the input $x$ and let $\eta = y / \|y\|_1$ describe a probability distribution over the labels. Our loss function is defined as follows:
\begin{align}\label{eq.loss_multilabel}
 \mathcal{L}(z, y) = & \|y\cdot(\ourr(z, r) - \eta)\|^2_2 + \sum_{y_i=1,y_j=0} \max\left(0,\eta_i - (z_i-z_j)\right),
\end{align}
where $y_i$ is $i$-th coordinate of the vector $y$ (similarly for $z$ and $\eta$). The first term focuses on approximating the probability on positive labels $\eta_i$ by $\ourr(z,r)_i$. The second term is responsible for pushing the logits of negative labels away from the positive ones by the margin $\eta_i$.

\begin{figure*}[t!]
\vskip 0.1in
\begin{center}
    \includegraphics[width=0.9\textwidth,height=.09\textwidth]{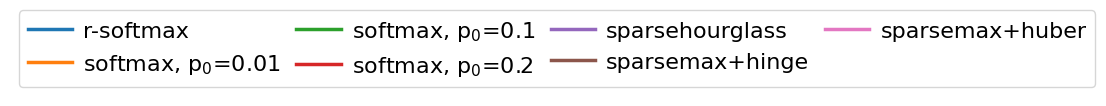}\hfill
    \begin{subfigure}[b]{\textwidth}
    \centering  
    \includegraphics[width=.32\textwidth,height=.3\textwidth]{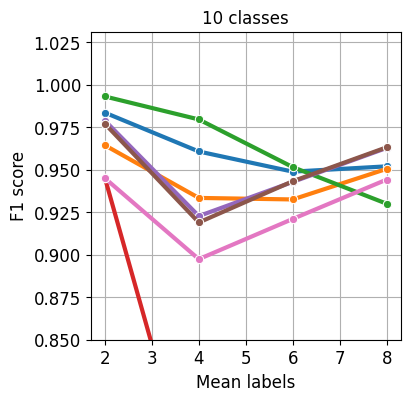}
    \includegraphics[width=.32\textwidth,height=.3\textwidth]{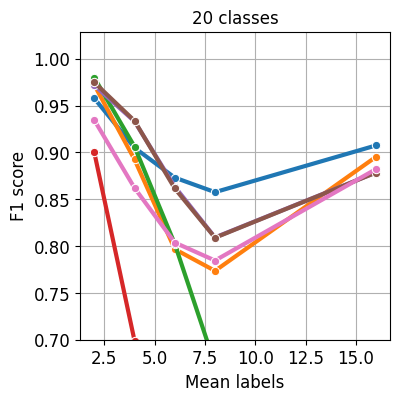}
    \includegraphics[width=.32\textwidth,height=.3\textwidth]{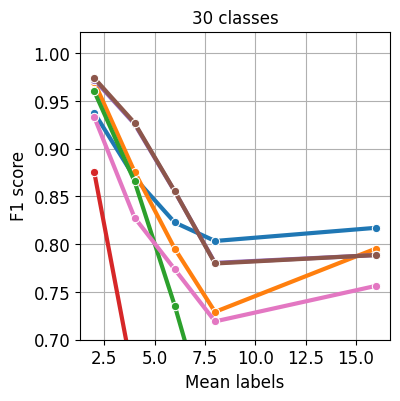}
    \caption{Varying average number of positive labels.}
    \label{fig:synthtask1}
    \end{subfigure}
    \vskip 0.1in
    \begin{subfigure}[b]{\textwidth}
    \centering  
    \includegraphics[width=.32\textwidth,height=.3\textwidth]{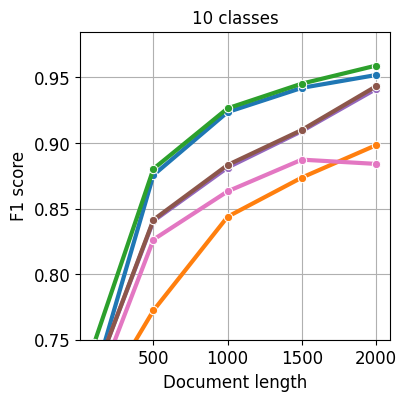}
    \includegraphics[width=.32\textwidth,height=.3\textwidth]{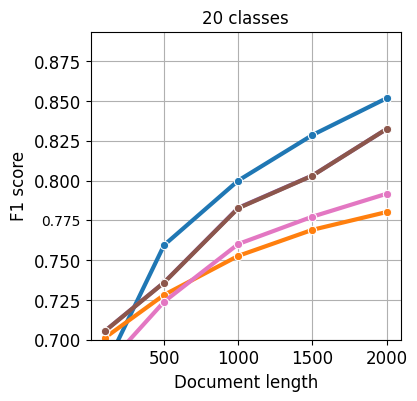}
    \includegraphics[width=.32\textwidth,height=.3\textwidth]{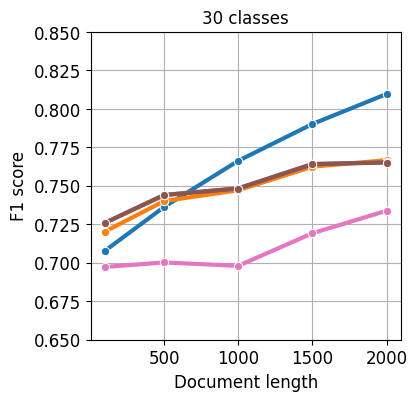}
    \caption{Varying document length. \label{fig:synthtask2}}
    \end{subfigure}
\end{center}
\caption{Different probability mapping functions for multi-label classification on the various synthetic datasets for different possible output class number (10, 20, 30). For datasets with fewer output classes (plots on the left) all functions produce similar results. However, for datasets with larger number of output classes (graphs in the middle and right) $\ourr$~seems to be the most beneficial choice.\label{fig:synthtask}}
\vskip -0.2in
\end{figure*}

\subsubsection{Datasets}
As preliminary experiments, we study a multi-label classification problem on synthetic data generated similarly to~\cite{laha2018controllable} using the scikit-learn library\footnote{\url{https://scikit-learn.org/stable/modules/generated/sklearn.datasets.make_multilabel_classification.html}}. We evaluate different probability mapping functions on varying average number of labels per sample (the document length is fixed at 2000) and on a different average document length which is the sum of the features per sample (in this case, the average number of labels is fixed at half the number of output classes). More specifically, these parameters are the expected values for Poisson distribution. Generated datasets consist of 5000 samples with 128 features, where 80\% of the data is the training set and 20\% is the validation set. We conducted experiments for 10, 20, and 30 possible output classes. 

Finally, we analyze the performance of considered functions on multi-label classification task on two popular real datasets: VOC 2007~\cite{pascalvoc2007} and COCO~\cite{cocodataset}. For these datasets, we resize the images to a height and width of 224, scale them to $[0, 1]$, and then normalize each channel. 

\subsubsection{Experimental setting}
As a baseline, we consider multi-label classification model with probability mapping function given by other sparse softmax alternatives such as sparsemax~\cite{sparsemax}, and sparsehourglass~\cite{laha2018controllable}. We assume that all non-zero values mean that the model predicted membership to the given class. For completeness, we also report the results of typical softmax~\cite{bridle1990probabilistic}. Theoretically, it is impossible to get zero values using softmax function (in practice, this can happen due to floating point precision), so we perform a search through various thresholds $p_0$ below which we consider the model to recognize class as negative.

For softmax function we use cross-entropy as a loss function, for sparsehourglass we use the cost function proposed by~\cite{laha2018controllable} and for sparsemax we test two functions, the one proposed originally by the authors~\cite{sparsemax}~(sparsemax$+$huber) and the one proposed by~\cite{laha2018controllable}~(sparsemax$+$hinge).



We use a simple two-layers neural network for synthetic datasets and pre-trained ResNet models~\cite{resnet} for real datasets (Resnet18 for VOC and Resnet101 for COCO) with an additional linear layer for classification followed by an evaluated activation function. We train the models with a learning rate~$\lambda=10^{-3}$ for synthetic datasets and with~$\lambda\in\{10^{-3}, 10^{-4}, 10^{-5}\}$ for VOC and COCO. For all scenarios, we use the Adam algorithm for gradient-based optimization~\cite{adam}.

Our $\ourr$ is parameterized by the sparsity rate~$r$, which corresponds to the desired fraction of zero labels in the multi-label experiment. To find its optimal value, we add an additional layer to the neural network which is responsible for predicting the sparsity rate that is later passed as an argument to $\ourr$ function. We supplied the multi-label classification cost function with the cross-entropy loss component responsible for evaluating the correctness of the number of labels indicated by the model.

In all settings, we report the best results on the validation set after the models achieve stability in the results on the validation set. For synthetic datasets, we train models for 150 epochs, and on VOC and COCO datasets we train models for 100 epochs. We use the F1 score as the quality metric for the multi-label classification models as it operates on the returned classes rather than on target scores (e.g., mean average precision metric).

\begin{table}[htb]
\caption{Effect of using different probability mapping functions for the multi-label classification problem for VOC and COCO validation datasets. Our function $\ourr$ (our) performs better than other tested sparse probability mapping functions (sparsemax and sparsehourglass) and it is also competitive to softmax itself, which requires the additional selection of a class indication threshold.\label{tab:multireal}}
\vskip 0.15in
\begin{center}
\begin{small}
\begin{tabular}{lll}
\toprule
Experimental setup & \makecell[l]{VOC \\ (F1)} & \makecell[l]{COCO \\ (F1)}\\
\midrule
Softmax ($p_0$=0.05) & 75.05 & 71.38 \\
Softmax ($p_0$=0.10) & 78.87 & 72.29  \\
Softmax ($p_0$=0.15) & \textbf{79.43} & 69.22  \\
Softmax ($p_0$=0.20) & 79.07 & 64.88  \\
Softmax ($p_0$=0.30) & 75.88 & 54.76  \\
\midrule
Sparsemax$+$huber & 66.84 & 52.30 \\
Sparsemax$+$hinge & 71.91 & 65.67  \\
Sparsehourglass & 71.35 & 64.85  \\
$\ourr$ & 77.90 & \textbf{72.56} \\
\bottomrule
\end{tabular}
\end{small}
\end{center}
\vskip -0.1in
\end{table}

\begin{figure*}[htb]
\vskip 0.1in
\begin{center}
    \hskip 0.15in
    \includegraphics[width=.9\textwidth,height=.09\textwidth]{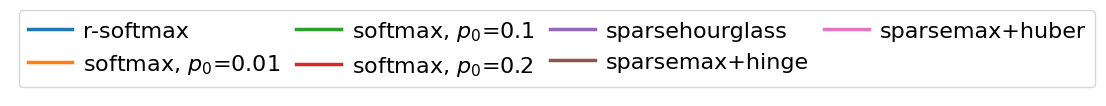} \hfill
    \includegraphics[width=.45\textwidth,height=.45\textwidth]{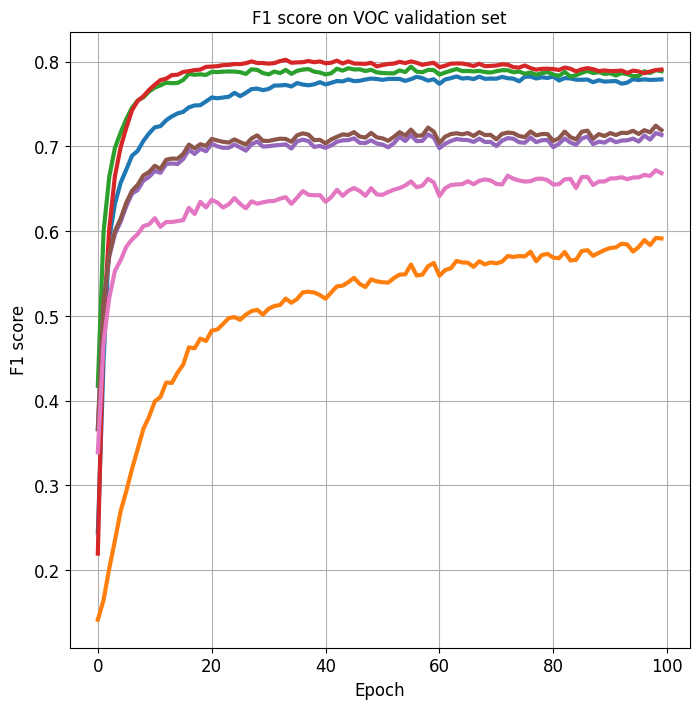} \hskip 0.2in
    \includegraphics[width=.45\textwidth,height=.45\textwidth]{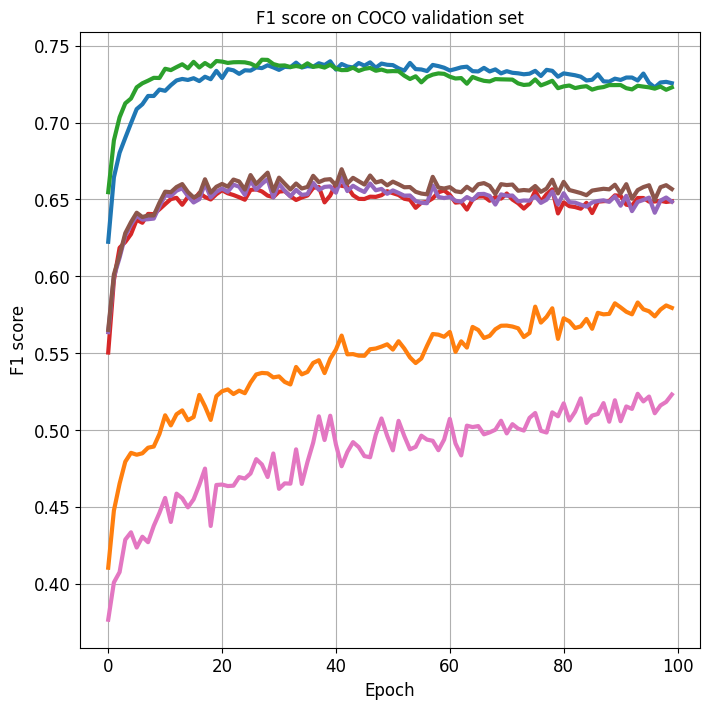}\hfill
\caption{Learning process (F1 score) when using different probability mapping functions for multi-label classification on VOC and COCO validation datasets. \label{fig:voccocolrc}}
\end{center}
\vskip -0.1in
\end{figure*}

\subsubsection{Results on synthetic datasets} 
\Cref{fig:synthtask} presents the performance of $\ourr$ function and its competitors (softmax, sparsemax, sparsehourglass) for multi-label classification experiments on the synthetic data validation set. For clarity of the graphs, we truncate the y-axis, omitting the notably lower results achieved by specific softmax versions with a particular $p_0$. 

In \Cref{fig:synthtask1} we compare the model behavior depending on the average number of positive labels. We can observe that all functions produce similar results for a small number of positive labels on average. However, for increasing the average number of positive labels, we may notice that our method produces the best results, especially when the dataset has a large number of possible output classes. 

In \Cref{fig:synthtask2} we show the impact of the average document length in data. In these experiments, we can also observe the superior or comparable performance of $\ourr$ for most configurations. Similarly like previously, we may observe that for a small number of classes in the output, our method is comparable to other functions. However, for a larger number of possible output classes, our method obtains the best results. Please note that the results for softmax with $p_0\in\{0.1, 0.2\}$ and output classes 20 and 30 are not included in the plots as they produce significantly worse results. This can be caused by the fact that for a larger number of output classes, probabilities are distributed over more components. This may lead to a situation where the output values are very small and it is more difficult to choose the appropriate threshold. Taking into consideration both of these experiments we conclude that in the investigated scenarios our method is the preferred choice as it generally provides the most benefits.

\subsubsection{Results on real datasets} We also evaluate $\ourr$ on real, multi-label classification datasets VOC and COCO, see~\Cref{tab:multireal}. Our $\ourr$ outperforms other sparse softmax alternatives and is very competitive with the original softmax. Although the performance of $\ourr$ is comparable to specific parametrization of softmax, the model with softmax requires the selection of appropriate threshold $p_0$ to indicate positive labels.  In practice, such selection has to be performed on the validation set, which generates additional computational costs. 

Additionally, in~\Cref{fig:voccocolrc} we report the F1 score learning curves for these experiments to observe how the model performance changes during learning depending on the considered probability mapping function. On the plots, we may observe  that model with $\ourr$ is learning much better than models with other sparse alternatives. Some softmax versions with a particular threshold converge faster than $\ourr$, but this most likely happens because the model has to learn the appropriate sparsity rate $r$, which requires a little more time. An advantage, however, is that there is no need to adjust any further thresholds afterward.

\subsection{Alternative to softmax in the self-attention block in transformer-based model}
\label{sec:expattn}

Nowadays, models based on the transformer architecture~\cite{vaswani2017attention} are the foundation for many state-of-the-art solutions in different fields, including natural language processing (NLP). A core element of the transformer is the attention mechanism, which is responsible for identifying important information for the neural network. In general, an attention block produces output based on input vectors: queries, keys, and values. The output is a sum of weighted values, where each weight is determined based on a query and corresponding key. For efficient computations, sets of queries, keys, and values are combined into matrices Q, K, and V. 

In more detail, each layer of the transformer contains a self-attention module, which is designed to indicate which tokens (parts of words in the text) in a sequence are contextually relevant to other tokens of the same sequence. Each of the self-attention blocks applies a softmax function that maps the resulting vector of the scaled dot product of queries $Q$ and keys $K$ of dimension $d_k$ into probabilities that represent weights for all values $V$, as shown below:

\begin{equation}
    Attention(Q, K, V ) = softmax(\frac{QK^T}{\sqrt{d_k}})V.
    \label{eq:selfattn}
\end{equation}
It is worth noting here that using softmax in this formula imposes an assignment of non-zero weight to each of the tokens in the sequence. In other words, every token, even insignificant one, has to be taken into account in further calculations.

In this section, we will demonstrate that replacing the softmax function with $\ourr$ that can return a sparse probability distribution is beneficial, as the model is able to ignore irrelevant tokens in the sequence.

\subsubsection{Experimental setting}
In our experiments we use a pre-trained transformer language model BERT~\cite{devlin-etal-2019-bert}, in which we focus on the probability mapping function in each of the self-attention blocks while fine-tuning the model. We report the performance of the baseline scenario with softmax as well as with its replacements: sparsemax, sparsehourglass, and $\ourr$. The implementation is based on the transformers library from Huggingface~\cite{wolf-etal-2020-transformers}.

We evaluate BERT model versions on several GLUE benchmark classification tasks~\cite{glue}, namely MRPC, RTE, SST-2, QNLI and QQP. We fine-tune the model for 5 epochs for the MRPC task and for 3 epochs for the other tasks. We report the final score on the validation datasets. We test different values of a learning rates for all models $\lambda \in \{10^{-5}, 2\cdot10^{-5}, 5\cdot10^{-5}, 10^{-4}, 5\cdot10^{-4}\}$.

Since we would like to check several possible final sparsity rates $r$ for $\ourr$, we linearly increase the hyperparameter $r$ during training from 0 (dense output) to the desired sparsity $r \in \{0.05, 0.1, 0.15, 0.2, 0.5\}$. During preliminary experiments, we observed that linear increase of the zeros fraction has its benefits, as the model has time to adapt to a given sparsity rather than losing information all at once.



\subsubsection{Results}
\Cref{tab:glue} summarizes results for different GLUE downstream tasks obtained by the best run in the grid search described in the previous section. We may observe that in most cases, applying $\ourr$ instead of the softmax function improves the performance of the fine-tuned transformer-based model. Other sparse alternatives like sparsemax and sparsehourglass have demonstrated poor performance in this application.

\begin{table*}[tb!]
\caption{Using different probability mapping functions in self-attention blocks of pretrained BERT language model. We report results after finetuning a model on several GLUE benchmark tasks. Our $\ourr$, introducing a specific sparsity level, outperforms other proposals.\label{tab:glue}}
\vskip 0.15in
\begin{center}
\begin{small}
\begin{tabular}{llllll}
\toprule
Experiment setup & \makecell[l]{MRPC \\ (Acc)} & \makecell[l]{RTE \\ (Acc)} & \makecell[l]{SST-2 \\ (Acc)} & \makecell[l]{QNLI \\ (Acc)}  & \makecell[l]{QQP \\ (Acc)} \\
\midrule
Softmax & 84.56 & 68.95 & 92.32 & \textbf{91.76} & 91.12 \\
Sparsemax & 68.38 & 52.71 & 79.82 & 55.57 & 77.18\\
Sparsehourglass & 68.38 & 52.71 & 79.24 & 70.99 & 76.04\\
$\ourr$ & \textbf{85.54} & \textbf{71.84} & \textbf{92.89} & 91.73 & \textbf{91.13} \\
\bottomrule
\end{tabular}
\end{small}
\end{center}
\vskip -0.1in
\end{table*}


We examined $\ourr$ performance for various final sparsity rates. We linearly increased the sparsity rate from $r=0$ until it reached the desired value. The gradual incorporation of sparsity is intended to give the model time to adapt to the changes. We found that introducing only a small sparsity (small $r$) into the self-attention output produces the best results while enforcing too many zeros (large $r$) makes the results worse. The best performance for tasks QQP, MRPC, QNLI, RTE and SST-2 was achieved by $r={0.1, 0.15, 0.15, 0.2, 0.2}$ respectively. Results suggest that in general it is beneficial for the model to eliminate distracting elements that are irrelevant to the considered sample. However, excluding a larger number of elements (by zeroing their importance) is not advantageous because either the model loses too much context or because the gradient flow during learning becomes more challenging.

\section{Conclusions}
In this paper, we proposed $\ourr$, a generalization of softmax, producing sparse probability distribution with a controllable sparsity rate. We applied $\ourr$ as an output layer in the multi-label classification problem and as a scoring function in the self-attention module used in NLP tasks. The obtained results confirm that in most cases $\ourr$ is highly competitive or superior to baseline softmax and other sparse probability mapping functions. Furthermore, $\ourr$ offers a more intuitive representation of the data, that is adjustable by one simple parameter determining what fraction of the data should be zero.



\subsubsection{Acknowledgements} 
The work of Klaudia Bałazy and Łukasz Struski was supported by the National Centre of Science (Poland) Grant No. 2020/39/D/ST6/
01332. The research of Jacek Tabor was carried out within the research project "Bio-inspired artificial neural network" (grant no. POIR.04.04.00-00-14DE/18-00) within the Team-Net program of the Foundation for Polish Science co-financed by the European Union under the European Regional Development Fund. The work of Marek Śmieja was supported by the National Centre of Science (Poland) Grant No. 2022/45/B/ST6/01117. Klaudia Bałazy is affiliated with Doctoral School of Exact and Natural Sciences at the Jagiellonian University. 

%
%
%
\bibliographystyle{splncs04}
\bibliography{iccs23.bib}
%




\end{document}